\documentclass[letterpaper]{article} 
\usepackage[draft]{aaai2026}  
\usepackage{times}  
\usepackage{helvet}  
\usepackage{courier}  
\usepackage[hyphens]{url}  
\usepackage{graphicx} 
\usepackage{natbib}  
\usepackage{caption} 
\usepackage{algorithm}
\usepackage{algorithmic}

\usepackage{bm}
\usepackage{amsfonts}
\usepackage{amsmath}
\usepackage{booktabs}
\usepackage{mathtools}
\usepackage{amsthm}

\usepackage{times}
\usepackage{helvet}
\usepackage{courier}
\usepackage{xcolor}

\newtheorem{proposition}{Proposition}

\newtheorem{assumption}{Assumption}

\allowdisplaybreaks

\usepackage{newfloat}
\usepackage{listings}
\DeclareCaptionStyle{ruled}{labelfont=normalfont,labelsep=colon,strut=off} 
\floatstyle{ruled}
\newfloat{listing}{tb}{lst}{}
\floatname{listing}{Listing}
\title{Adaptive Batch Size and Learning Rate Scheduler for Stochastic Gradient Descent Based on Minimization of Stochastic First-order Oracle Complexity}
\author{
 Hikaru Umeda,
 Hideaki Iiduka
}
\affiliations {
Meiji University \\
ee227115@meiji.ac.jp, iiduka@cs.meiji.ac.jp
}

\usepackage{bibentry}

\begin{document}

\maketitle

\begin{abstract}
The convergence behavior of mini-batch stochastic gradient descent (SGD) is highly sensitive to the batch size and learning rate settings. Recent theoretical studies have identified the existence of a critical batch size that minimizes stochastic first-order oracle (SFO) complexity, defined as the expected number of gradient evaluations required to reach a stationary point of the empirical loss function in a deep neural network. An adaptive scheduling strategy is introduced to accelerate SGD that leverages theoretical findings on the critical batch size. The batch size and learning rate are adjusted on the basis of the observed decay in the full gradient norm during training. Experiments using an adaptive joint scheduler based on this strategy demonstrated improved convergence speed compared with that of existing schedulers.  
\end{abstract}

\begin{links}
    \link{Code}{https://anonymous.4open.science/r/adaptive-scheduler}
\end{links}

\section{Introduction}
The rapid increase in the computational cost of training deep neural networks (DNNs) has made efficient optimization strategies more important than ever. Mini-batch stochastic gradient descent (SGD) \citep{robb1951,zinkevich2003,nem2009,gha2012,gha2013} and its variants are widely used due to their simplicity and scalability. However, the convergence behavior of these methods is highly sensitive to hyperparameters such as batch size (BS) and learning rate (LR), especially in the nonconvex optimization landscapes characteristic of deep learning.

Among these hyperparameters, BS plays a particularly important role. Increasing the BS~\citep{Byrd:2012aa,balles2016coupling,pmlr-v54-de17a,l.2018dont,goyal2018accuratelargeminibatchsgd,shallue2019,zhang2019} has been shown to reduce the gradient variance and accelerate training. 

Recently reported results \citep{umeda2025increasing} indicate that effective LRs for SGD are either constant or increasing as BS is increased because increasing both BS and LR speeds SGD convergence. Hence, in this work, we focused on \emph{using an increasing BS and an increasing or constant LR} (as represented in \eqref{bs} and \eqref{lr}). 

Recent theoretical studies have highlighted the importance of \emph{stochastic first-order oracle (SFO) complexity} \citep{doi:10.1137/120880811, Ghadimi:2016aa}, defined as the expected number of gradient evaluations required to reach a stationary point of the empirical loss function in a DNN. A key insight from these studies is the existence of a \emph{critical BS} \citep{shallue2019,zhang2019,pmlr-v202-sato23b,Imaizumi19062024, tsukada2025relationship, sato2025analysismuonsconvergencecritical} that minimizes SFO complexity; increasing the BS beyond this point can actually degrade overall training efficiency due to increased per-iteration cost. Optimizers that operate at the critical BS converge more rapidly since they minimize SFO complexity.

We have developed a novel scheduler for mini-batch SGD that \textbf{adjusts the BS and LR} on the basis of the \emph{critical BS} at each training stage. The full gradient norm—defined as the norm of the empirical loss gradient—is used as a signal to  adjust the training schedule—with the aim of reducing SFO complexity while ensuring stable convergence.

\subsection{Contributions}
The contributions of this work are as follows:
\begin{itemize}
\item \textbf{Theoretical Foundation}: We provide a theoretical foundation for adaptive scheduling by showing that the critical BS required to minimize SFO complexity scales as $O(1/\epsilon^2)$, where $\epsilon$ denotes the threshold for the target full gradient norm (see Propositions \ref{prop:1} and \ref{prop:2}).
\item \textbf{Adaptive Scheduling Strategy}: We present a scheduling strategy that adaptively adjusts both BS and LR on the basis of the current full gradient norm (see \eqref{eq:linear-scheduler} and \eqref{eq:exp-scheduler}), and we demonstrate that this strategy accelerates SGD while guaranteeing convergence (Proposition \ref{prop:3}).
\item \textbf{Algorithm Design}: We present a practical adaptive algorithm that transitions between training stages when the full gradient norm falls below a predefined threshold and updates the hyperparameters accordingly (Algorithm \ref{algo:2}).
\item \textbf{Empirical Validation}: We demonstrate on the CIFAR-10 and CIFAR-100 datasets \citep{Krizhevsky2009} that our method accelerates convergence compared with baseline schedulers with fixed or periodic update rules.
\item \textbf{Comparison with Existing Methods}: We compare our approach with three commonly used scheduling strategies and show that it achieves superior performance (see Evaluation Section).
\end{itemize}

\section{Theoretical Background}
\subsection{Empirical Risk Minimization}
Let $\bm{\theta} \in \mathbb{R}^d$ denote the parameter of a DNN, where $\mathbb{R}^d$ is a $d$-dimensional Euclidean space with inner product $\langle \cdot, \cdot \rangle$ and norm $\|\cdot\|$. Let $S = \{(\bm{x}_1,\bm{y}_1), \cdots, (\bm{x}_n,\bm{y}_n)\}$ denote the training set, where $n \in \mathbb{N}$ is the number of 
samples, and each data point $\bm{x}_i$ is paired with label $\bm{y}_i$. Let $f_i (\cdot) := f(\cdot;(\bm{x}_i,\bm{y}_i)) \colon \mathbb{R}^d \to \mathbb{R}_+$ denote the loss function corresponding to the $i$-th labeled training data $(\bm{x}_i,\bm{y}_i)$. 
Our objective is to solve the empirical risk minimization problem by minimizing the empirical loss, 
defined for all $\bm{\theta} \in \mathbb{R}^d$ as
\begin{align}\label{er}
    f (\bm{\theta}) 
    \coloneqq \frac{1}{n} \sum_{i=1}^n f(\bm{\theta};(\bm{x}_i,\bm{y}_i))
    = \frac{1}{n} \sum_{i=1}^n f_i(\bm{\theta}).
\end{align}
We assume that the loss function $f_i$ ($i\in [n] \coloneqq \{1,\cdots, n\}$) 
satisfies the conditions stated in the following standard assumption.

\begin{assumption}\label{assum:1}
Let $L > 0$ and $\sigma \geq 0$.
\begin{description}
\item[(A1)] Each loss function $f_i \colon \mathbb{R}^d \to \mathbb{R}$ is continuously differentiable. Moreover, the empirical loss $f$ defined in \eqref{er} is $L$-smooth; that is, for all $\bm{\theta}_1, \bm{\theta}_2 \in \mathbb{R}^d$,
    $\|\nabla f(\bm{\theta}_1) - \nabla f(\bm{\theta}_2)\| \leq L \|\bm{\theta}_1 - \bm{\theta}_2\|$.
In addition, we assume that the minimal value of $f$ is finite; i.e., \( f^\star \coloneqq \min_{\bm{\theta} \in \mathbb{R}^d} f(\bm{\theta}) \in \mathbb{R} \).

\item[(A2)] Let $f_{\xi} \colon \mathbb{R}^d \to \mathbb{R}$ denote a loss function randomly selected from the set $\{f_1, \cdots, f_n\}$, where \(\xi\) is a random variable independent of \(\bm{\theta} \in \mathbb{R}^d\). The stochastic gradient of $\nabla f$, $\nabla f_\xi$, satisfies the following conditions:
  \begin{align*}
    &\textnormal{(i)} \quad \mathbb{E}_{\xi}[\nabla f_{\xi}(\bm{\theta})] = \nabla f(\bm{\theta}), \\
    &\textnormal{(ii)} \quad
    \mathbb{V}_\xi [\nabla f_\xi (\bm{\theta})] = \mathbb{E}_{\xi}\left[\| \nabla f_{\xi}(\bm{\theta}) - \nabla f(\bm{\theta}) \|^2\right] \leq \sigma^2,
  \end{align*}
where $\mathbb{E}_\xi [X]$ (resp. $\mathbb{V}_\xi [X]$) denotes the expectation (resp. variance) of $X$ with respect to $\xi$. 
\end{description}
\end{assumption}
When the random variable $\xi$ follows a discrete uniform distribution $\mathrm{DU}(n)$—as is standard in stochastic training of DNNs, it is obvious that condition (A2)(i) holds. That is, the stochastic gradient $\nabla f_\xi$ is an unbiased estimator of the full gradient $\nabla f$. Furthermore, suppose that each component function $f_i$ is $L_i$-smooth over a compact set $C$ (e.g., a closed ball centered at the origin $\bm{0}$ with sufficiently large radius $R$). Then the $L$-smoothness of $f$ in (A1) with $L = \frac{1}{n} \sum_{i\in [n]} L_i$, and (A2)(ii) with $\sigma^2 = \frac{2}{n} \sum_{i\in [n]} L_i (f^{\star \star} - f^\star)$ holds, where $f^{\star \star} \coloneqq \max_{\bm{\theta} \in C} f(\bm{\theta})$ (see, e.g., \citep[Appendix A.1]{umeda2025increasing} for a detailed derivation).

\subsection{Mini-batch SGD}
Given the $t$-th approximated point $\bm{\theta}_t \in \mathbb{R}^d$, mini-batch SGD uses $b_t$ loss functions $f_{\xi_{t,1}},\cdots,f_{\xi_{t,b_t}}$ randomly chosen from $\{f_1,\cdots,f_n\}$, where $\bm{\xi}_t \coloneqq (\xi_{t,1}, \cdots, \xi_{t,b_t})^\top$ consists of $b_t$ independent and identically distributed variables and $\bm{\xi}_t$ is independent of $\bm{\theta}_t$. The mini-batch gradient is defined by

\begin{align}\label{mini_batch_grad}
    \nabla f_{B_t} (\bm{\theta}_t) \coloneqq \frac{1}{b_t} \sum_{i=1}^{b_t} \nabla f_{\xi_{t,i}}(\bm{\theta}_t),
\end{align} 
where sample size $b_t \in \mathbb{N}$ is the BS. Mini-batch SGD updates the $(t+1)$-th approximated point as $\bm{\theta}_{t+1} \coloneqq \bm{\theta}_t - \eta_t \nabla f_{B_t} (\bm{\theta}_t)$, where $\eta_t > 0$ is the LR. The pseudo-code of mini-batch SGD is shown as Algorithm \ref{algo:1}.

\begin{algorithm}
\caption{Mini-batch SGD}
\label{algo:1}
\begin{algorithmic}[1]
\REQUIRE
$\bm{\theta}_0 \in \mathbb{R}^d$ (initial point), 
$b_t > 0$ (batch size), 
$\eta_t > 0$ (learning rate), 
$T \geq 1$ (steps).
\ENSURE 
$(\bm{\theta}_t) \subset \mathbb{R}^d$
\FOR{$t=0,1,\ldots,T-1$}
\STATE{ 
$\nabla f_{B_t}(\bm{\theta}_t)
:=
\frac{1}{b_t} \sum_{i=1}^{b_t} \nabla f_{\xi_{t,i}}(\bm{\theta}_t)$}
\STATE{
$\bm{\theta}_{t+1} 
:= \bm{\theta}_t - \eta_t \nabla f_{B_t}(\bm{\theta}_t)$}
\ENDFOR
\end{algorithmic}
\end{algorithm}

Assumption (A2)(i) implies that mini-batch gradient $\nabla f_{B_t}(\bm{\theta}_t)$, defined in \eqref{mini_batch_grad}, is an unbiased estimator of the full gradient $\nabla f (\bm{\theta}_t)$, and Assumption (A2)(ii) implies that the variance of the mini-batch gradient $\nabla f_{B_t}(\bm{\theta}_t)$, defined in \eqref{mini_batch_grad}, is bounded above. That is, the mini-batch gradient $\nabla f_{B_t}(\bm{\theta}_t)$ satisfies the following conditions:
\begin{align}\label{conditions_mini_batch}
    \mathbb{E}_{\bm{\xi}_t} \left[ \nabla f_{B_t} (\bm{\theta}_t)  \right] = \nabla f (\bm{\theta}_t) \text{ and } 
    \mathbb{V}_{\bm{\xi}_t} \left[ \nabla f_{B_t} (\bm{\theta}_t) \right] \leq \frac{\sigma^2}{b_t},
\end{align}
where these conditions hold under the assumption that $\bm{\xi}_t$ is independent of the history $[\bm{\xi}_{t-1}] \coloneqq \{ \bm{\xi}_0, \cdots, \bm{\xi}_{t-1} \}$. Using the condition $\mathbb{E}_{\bm{\xi}_t} \left[ \nabla f_{B_t} (\bm{\theta}_t) \right] = \nabla f (\bm{\theta}_t)$, the search direction $\bm{d}_t \coloneqq - \nabla f_{B_t}(\bm{\theta}_t)$ in mini-batch SGD satisfies 
\begin{align*}
    \mathbb{E} \left[ \left\langle \nabla f (\bm{\theta}_t), \bm{d}_t \right\rangle \right] = - \mathbb{E} \left[ \| \nabla f(\bm{\theta}_t) \|^2 \right] < 0,
\end{align*}
where $\mathbb{E}$ denotes the total expectation defined by $\mathbb{E} \coloneqq \mathbb{E}_{\bm{\xi}_0} \cdots \mathbb{E}_{\bm{\xi}_t}$, and we assume $\nabla f (\bm{\theta}_t) \neq \bm{0}$. That is, the search direction $\bm{d}_t \coloneqq - \nabla f_{B_t}(\bm{\theta}_t)$ is a descent direction of $f$, as defined in \eqref{er}, in the sense of the total expectation. It is expected that mini-batch SGD (Algorithm \ref{algo:1}), using the descent direction $\bm{d}_t \coloneqq - \nabla f_{B_t}(\bm{\theta}_t)$, finds a local minimizer of the empirical loss $f$ defined in \eqref{er}. Therefore, we focus on finding a stationary point $\bm{\theta}^\star \in \mathbb{R}^d$ of $f$ such that $\nabla f(\bm{\theta}^\star) = \bm{0}$.

\subsection{Upper Bound of Full Gradient Norm Generated by Mini-batch SGD}
Let $\eta_t$ $(\in [\eta_{\min}, \eta_{\max}] \subset [0, \frac{2}{L}))$ satisfy the condition $\sum_{t=0}^{T-1} \eta_t \neq 0$. Under Assumption \ref{assum:1}, the total expectation of the full gradient norm $\mathbb{E}[\|\nabla f (\bm{\theta}_t)\|]$ generated by mini-batch SGD satisfies the following bound from \citep[Lemma 2.1]{umeda2025increasing}: for all $T \in \mathbb{N}$, 
\begin{align}\label{ineq_1}
\min_{t\in [0:T-1]} \mathbb{E} \left[\|\nabla f(\bm{\theta}_t)\| \right]
\leq
\sqrt{B_T + V_T},
\end{align}
where $[0:T-1] \coloneqq \{0,1,\cdots, T-1\}$, and the bias term $B_T$ and the variance term $V_T$ are defined as follows:
\begin{align}
    &B_T \coloneqq \frac{2(f(\bm{\theta}_0) - f^\star)}{2 - L \eta_{\max}}
\frac{1}{\sum_{t=0}^{T-1} \eta_t},\label{B}\\   
    &V_T \coloneqq \frac{L \sigma^2}{2 - L \eta_{\max}} 
\frac{1}{\sum_{t=0}^{T-1} \eta_t} \sum_{t=0}^{T-1} \frac{\eta_t^2}{b_t}.\label{V}
\end{align}
Inequality \eqref{ineq_1} follows from the conditions in \eqref{conditions_mini_batch} and the descent lemma, which holds under the $L$-smoothness of $f$ in (A1). This inequality implies that, if both the bias term $B_T$ and the variance term $V_T$ converge to $0$ as $T \to + \infty$, then mini-batch SGD converges to a stationary point of $f$. The convergence behavior of $B_T$ and $V_T$, as defined in \eqref{B} and \eqref{V}, depends critically on the BS $b_t$ and LR $\eta_t$ settings.

\subsubsection{Batch size}
We consider BS defined as
\begin{align}\label{bs}
    b_m =
    \begin{dcases}
        b_0 + m \Delta b &\text{(Linearly increasing BS)}\\
        b_0 \delta^m &\text{(Exponentially increasing BS)},
    \end{dcases}
\end{align}
where $b_0$ is the initial BS, $m \in [0:M]$ denotes a stage  during which BS is kept constant, $\Delta b > 0$ is the increment in BS per stage, and $\delta > 1$ is the scaling factor for BS per stage. Let $T_m$ be the number of steps during stage $m$. Then, the total number of steps is $T = \sum_{m=0}^M T_m$. For example, under exponentially increasing conditions, BS is multiplied by $\delta$ per stage, and BS in stage $m$ is kept at $b_t = b_0 \delta^m$ $(t \in [T_m])$. 

The simplest BS is constant, $b_t =b_m = b$. The convergence of $V_T$ to $0$ depends on the setting of $\eta_t$ satisfying $\frac{\sum_{t=0}^{T-1} \eta_t^2}{\sum_{t=0}^{T-1} \eta_t} \to 0$ $(T \to + \infty)$. For example, a decaying LR $\eta_t = \frac{\eta_{\max}}{\sqrt{t+1}}$ satisfies $\frac{\sum_{t=0}^{T-1} \eta_t^2}{\sum_{t=0}^{T-1} \eta_t} \leq O( \frac{\log T}{\sqrt{T}}) \to 0$ $(T \to + \infty)$. However, the convergence rate $O(\frac{\log T}{\sqrt{T}})$ is slow. Meanwhile, increasing BS either linearly or exponentially yields a faster convergence rate than $O(\frac{\log T}{\sqrt{T}})$:
\begin{align}\label{V_T_inc_bs}
    V_T 
    &\leq \frac{L \sigma^2}{2 - L \eta_{\max}} \frac{1}{\sum_{m=0}^{M} \sum_{t=1}^{T_m} \eta_{\min}} \sum_{m=0}^{M} \sum_{t=1}^{T_m} \frac{\eta_{\max}^2}{b_t}\\
    &= \frac{L \sigma^2}{2 - L \eta_{\max}} \frac{\eta_{\max}^2}{\eta_{\min} T} \underbrace{\sum_{m=0}^M \sum_{t=1}^{T_m} \frac{1}{b_t}}_{\leq B < + \infty \text{ } (M \to + \infty)}
    = O \left(\frac{1}{T} \right). \nonumber
\end{align}
Hence, we focus on BS defined by \eqref{bs} as it ensures fast convergence of mini-batch SGD.

\subsubsection{Learning rate}
We consider LR defined as
\begin{align}\label{lr}
    \eta_m = 
    \begin{dcases}
        \eta &\text{ (Constant LR)}\\
        \eta_0 \gamma^m &\text{ (Exponentially increasing LR)},
    \end{dcases}
\end{align}
where $\eta \in (0, \frac{2}{L})$, $m \in [0:M]$ is a stage index such that BS and LR are kept constant (see \eqref{bs}), $\eta_0$ is the initial LR, and $\gamma > 1$ satisfies $\gamma^2 < \delta$ ($\delta > 1$ is used in exponentially increasing BS). When LR is constant, $\eta_t = \eta_m = \eta$, we have 
\begin{align}\label{B_T_constant}
    B_T 
    = \frac{2(f(\bm{\theta}_0) - f^\star)}{2 - L \eta}\frac{1}{\eta T}
    = O \left(\frac{1}{T} \right).
\end{align}
Since BS defined by \eqref{bs} and a constant LR satisfy \eqref{V_T_inc_bs} with $\eta = \eta_{\max} = \eta_{\min}$, we also have $V_T = O (\frac{1}{T})$. When LR is increased exponentially, we have 
\begin{align}\label{B_T_exp}
    B_T 
    = \frac{2(f(\bm{\theta}_0) - f^\star)}{2 - L \eta_{\max}}\frac{1}{\sum_{m=0}^M \sum_{t=1}^{T_m} \eta_t}
    = O \left(\frac{1}{\gamma^T} \right).
\end{align}
Moreover, when BS is increased exponentially, as defined by \eqref{bs}, we have 
\begin{align}\label{V_T_exp}
\begin{split}
    V_T 
    &= 
    \frac{L \sigma^2}{2 - L \eta_{\max}} \frac{1}{\sum_{m=0}^M \sum_{t=1}^{T_m} \eta_t} \sum_{m=0}^M \sum_{t=1}^{T_m} \frac{\eta_t^2}{b_t}\\
    &\leq 
    O \Bigg( 
    \frac{1}{\gamma^T} \underbrace{\sum_{m=0}^M \left(\frac{\gamma^2}{\delta} \right)^m}_{D < + \infty \text{ } (M \to + \infty)} 
    \Bigg)
    = O \left( \frac{1}{\gamma^T} \right),
\end{split}
\end{align}
where the second inequality follows from $\frac{\gamma^2}{\delta} < 1$. From \eqref{B_T_exp} and \eqref{V_T_exp}, we need to set $\delta$ in \eqref{bs} and $\gamma$ in \eqref{lr} such that $\gamma < \sqrt{\delta}$ to guarantee fast convergence $O(\frac{1}{\gamma^T})$ of both $B_T$ and $V_T$ in mini-batch SGD.

\subsection{Convergence Rate of Mini-batch SGD}
The above discussion leads to the following proposition.

\begin{proposition}\label{prop:1}
Let $(\bm{\theta}_t)_{t=0}^{T}$ be the sequence generated by mini-batch SGD (Algorithm \ref{algo:1}) with $\eta_t$ $(\in (0, \frac{2}{L}))$ satisfying $\sum_{t=0}^{T-1} \eta_t \neq 0$ under Assumption \ref{assum:1}. Then, the following hold.

{\em (i)} Constant BS $b$ and Constant LR $\eta$:
\begin{align*}
    \min_{t\in [0:T-1]} \mathbb{E} \left[\|\nabla f(\bm{\theta}_t)\| \right]
    \leq\sqrt{\underbrace{\frac{2(f(\bm{\theta}_0) - f^\star)}{\eta (2 - L \eta)}}_{C_1} \frac{1}{T}
    + 
    \underbrace{\frac{L \eta \sigma^2}{2 - L \eta}}_{C_2} \frac{1}{b}}.
\end{align*}

{\em (ii)} Linearly increasing BS $b_m$ and Constant LR $\eta$:
\begin{align*}
    \min_{t\in [0:T-1]} \mathbb{E} \left[\|\nabla f(\bm{\theta}_t)\| \right]
    \leq 
    \sqrt{\frac{C_1}{T}
    + 
    \frac{B C_2}{T}}
    = 
    O \left( \frac{1}{\sqrt{T}} \right).
\end{align*}  

{\em (iii)} Exponentially increasing BS $b_m$ and 
LR $\eta_m$:
\begin{align*}
    \min_{t\in [0:T-1]} \mathbb{E} \left[\|\nabla f(\bm{\theta}_t)\| \right]
    = 
    O \left( \sqrt{\frac{C_1}{\gamma^T} + \frac{D C_2}{\gamma^T}} \right) = O \left(\frac{1}{\sqrt{\gamma^T}} \right).
\end{align*}
\end{proposition}

\begin{proof}
Property (i) follows from \eqref{ineq_1}, \eqref{B_T_constant}, and $V_T = \frac{L \sigma^2}{2 - L \eta} \frac{1}{\eta T} \frac{\eta^2 T}{b} = \frac{C_2}{b}$, Property (ii) follows from \eqref{ineq_1}, \eqref{V_T_inc_bs}, and \eqref{B_T_constant}, and Property (iii) follows from \eqref{ineq_1}, \eqref{B_T_exp}, and \eqref{V_T_exp}. 
\end{proof}

Let us compare the properties in Proposition \ref{prop:1}. For example, let $\delta = 2$ (i.e., BS is doubled at every stage; see \eqref{bs}). Then, we set $\gamma = 1.4 < \sqrt{2} = \sqrt{\delta}$ (i.e., LR is multiplied by $\gamma = 1.4$). Proposition \ref{prop:1}(iii) thus implies that mini-batch SGD with exponentially increasing BS and exponentially increasing LR achieves faster convergence $O(\frac{1}{\sqrt{\gamma^T}})$ than the $O(\frac{1}{\sqrt{T}})$ rate for the linearly increasing BS and constant LR scheduler in Proposition \ref{prop:1}(ii). The constant BS and LR scheduler in Proposition \ref{prop:1}(i) serves as a useful baseline for analyzing the $\epsilon$-approximation of mini-batch SGD discussed in the next subsection. 

\subsection{Minimization of SFO Complexity and Critical BS}
The case in which a DNN is trained using mini-batch SGD under an $\epsilon$-approximation is defined as 
\begin{align}\label{e_approx}
    \min_{t \in [0:T-1]} \mathbb{E}\left[\|\nabla f(\bm{\theta}_t)\| \right] \leq \epsilon,
\end{align}
where $\epsilon > 0$ denotes the target precision. First-order optimizers, such as SGD and its variants, rely on stochastic gradients estimated from mini-batches of training data. A fundamental metric in this context is \emph{SFO complexity}, defined as the total number of gradient computations required to achieve an $\epsilon$-approximation \eqref{e_approx}. When mini-batch SGD uses a constant BS $b$, the DNN model requires $b$ gradient evaluations per step. When $T$ is the number of steps required to achieve an $\epsilon$-approximation \eqref{e_approx}, 
\begin{align*}
    \boxed{\text{the SFO complexity } N \text{ is } b T.}
\end{align*}
We now consider the relationship between $N$, $T$, and $b$ for an $\epsilon$-approximation \eqref{e_approx} of mini-batch SGD. Proposition \ref{prop:1}(i) implies that mini-batch SGD with a constant BS $b$ and a constant LR $\eta$ satisfies
\begin{align}\label{upper_bound}
    \min_{t \in [0:T-1]} \mathbb{E}\left[\|\nabla f(\bm{\theta}_t)\| \right] 
    \leq 
    \underbrace{\sqrt{\frac{C_1}{T} + \frac{C_2}{b}}}_{ \leq \epsilon \text{ }  \Rightarrow \text{ } \eqref{e_approx}},
\end{align}
where $C_1$ and $C_2$ are positive constants defined as in Proposition \ref{prop:1}(i). If the upper bound in \eqref{upper_bound} is less than or equal to $\epsilon$, i.e., 
\begin{align}\label{T_e_approx}
    b > \frac{C_2}{\epsilon^2} \text{ and }
    T \geq \frac{C_1 b}{\epsilon^2 b - C_2} \eqqcolon T (b), 
\end{align}
then mini-batch SGD is an $\epsilon$-approximation \eqref{e_approx}. That is, if the number of steps achieves $T(b)$ defined by \eqref{T_e_approx}, which is a function of BS $b$, then mini-batch SGD is an $\epsilon$-approximation \eqref{e_approx}. Then, the SFO complexity needed to satisfy \eqref{e_approx} is 
\begin{align}\label{SFO_e_approx}
    \boxed{N(b) = b T(b) = \frac{C_1 b^2}{\epsilon^2 b - C_2}.}
\end{align}
This leads to the following proposition characterizing SFO complexity. 

\begin{proposition}\label{prop:2}
Let $\epsilon > 0$, and let $(\bm{\theta}_t)_{t=0}^{T}$ be the sequence generated by mini-batch SGD (Algorithm \ref{algo:1}) with a constant BS $b$ $(> \frac{C_2}{\epsilon^2})$ and a constant LR $\eta$ $(\in (0, \frac{2}{L}))$ under Assumption \ref{assum:1}. Then, $N(b)$ defined by \eqref{SFO_e_approx} is a convex function of BS $b$, and there exists a minimizer of $N(b)$ given by 
    \begin{align}\label{cbs}
    \text{Critical BS: } \boxed{b_{\epsilon}^\star 
    = \frac{2 C_2}{\epsilon^2}
    = O \left(\frac{1}{\epsilon^2} \right).}
    \end{align}
\end{proposition}

\begin{proof}
Under the assumptions in Proposition \ref{prop:2}, Proposition \ref{prop:1}(i) holds. Hence, $N(b)$ in \eqref{SFO_e_approx} is well-defined. We then have 
    \begin{align*}
        N'(b) = \frac{C_1 b (\epsilon^2 b - 2 C_2)}{(\epsilon^2 b - C_2)^2} 
        \text{ and }
        N''(b) = \frac{2 C_1 C_2^2}{(\epsilon^2 b - C_2)^3}.
    \end{align*}
Since $N''(b) \geq 0$, $N(b)$ is convex. Moreover, a minimizer \ $N(b)$ exists such that $N'(b_\epsilon^\star) = 0$; i.e., $\epsilon^2 b_\epsilon^\star - 2 C_2 = 0$, which implies that $b_\epsilon^\star$ is given as in \eqref{cbs}.
\end{proof}

We call the BS $b_{\epsilon}^\star$ that minimizes SFO complexity $N(b)$ a \emph{critical BS} (CBS). We can expect that mini-batch SGD using CBS has fast convergence since CBS minimizes the stochastic computational cost so that mini-batch SGD can be an $\epsilon$-approximation.

\section{Adaptive BS and LR Strategy}
We present an adaptive scheduling strategy for BS and LR that leverages theoretical findings on CBS. As shown in~\eqref{cbs}, the CBS $b_{\epsilon}^{\star}$ required to satisfy 
\eqref{e_approx} 
scales as $O(1/\epsilon^2)$. Reflecting this scaling behavior, $\epsilon$ is gradually decreased in multiple stages, and BS and LR are adjusted accordingly to match the corresponding critical values.

Formally, the number of stages $M$ (see \eqref{bs}) is fixed, and a sequence of decreasing target precisions is defined:
\begin{align}\label{e_m}
    \epsilon_0 > \cdots > \epsilon_m > \cdots >  \epsilon_{M-1}.
\end{align}
The target precision in stage $m$ is associated with the corresponding critical BS $b_m$ and LR $\eta_m$. Training begins with initial values $(\epsilon_0, b_0, \eta_0)$, where $b_0 = b_{\epsilon_0}^\star$ denotes the CBS that minimizes the SFO complexity needed to achieve an $\epsilon_0$-approximation using mini-batch SGD with a constant LR $\eta_0$. In practice, $b_0 = b_{\epsilon_0}^\star$ must be computed using SGD with a constant LR $\eta_0$; for example, Figure \ref{fig:cbs} shows that $\eta_0 = 0.1$ yields $b_{0.5}^\star = 2^4$ when training ResNet-18 on CIFAR-10. The full gradient norm is monitored throughout training. When it falls below $\epsilon_m$, the procedure transitions to the next stage $m+1$, and the training parameters are updated accordingly. The following describes how the target precision in \eqref{e_m} is set in accordance with Propositions \ref{prop:1} and \ref{prop:2}.

\subsection{Linearly Increasing BS and Constant LR Scheduler}
Proposition \ref{prop:1}(ii) establishes that the upper bound of $\min_{t \in [0:T-1]} \mathbb{E}[\|\nabla f(\bm{\theta}_t)\|]$ decays at a rate of $O(1/\sqrt{T})$ when BS is linearly increased and LR is kept constant. This means that, as training progresses and the full gradient norm decreases, the BS should be increased accordingly.

These observations support a scheduling strategy in which BS is increased in response to the decay of the full gradient norm. Specifically, we evaluated a scheduler with linearly increasing BS $b_m$ defined by \eqref{bs} and a constant LR $\eta_m = \eta$ defined by \eqref{lr} for stage $m$. The full gradient norm threshold $\epsilon_m$ is adjusted in accordance with the empirical decay pattern. Let $\epsilon_0 > 0$ be the initial target precision. The definition of a linearly increasing BS \eqref{bs} implies that BS $b_1$ for stage $1$ satisfies $2 \min\{ b_0, \Delta b \} \leq b_1 \leq 2 \max\{ b_0, \Delta b \}$. Meanwhile, from the definition of CBS \eqref{cbs}, CBS $b_{\epsilon_1}^\star$ for $\epsilon_1$-approximation is $b_{\epsilon_1}^\star = O (1/\epsilon_1^2)$, which implies $\epsilon_1 = O ( 1/\sqrt{b_{\epsilon_1}^\star})$. Assuming $\epsilon_1 < \epsilon_0$ in \eqref{e_m} yields $\epsilon_1 = \epsilon_0/\sqrt{2}$. By induction, $b_m = b_0 + m \Delta b = O(m+1)$ as defined by the linearly increasing BS in \eqref{bs}, and 
\begin{align*}
    \epsilon_m = \frac{\epsilon_0}{\sqrt{1 + m}}.
\end{align*}
Accordingly, we present a candidate scheduler, with parameters $b_0$, $\Delta b$, and $\eta$ as specified in \eqref{bs} and \eqref{lr}: 
\begin{align}\label{eq:linear-scheduler}
    &\textbf{[Linearly Increasing BS and Constant LR Scheduler]}\nonumber\\
    &\quad \boxed{b_m = b_0 + m \Delta b, \text{ }
    \eta_m = \eta, \text{ } 
    \epsilon_m = \frac{\epsilon_0}{\sqrt{1 + m}}.
    }
\end{align}
This scheduler aligns the increase in the BS with the theoretically required increase in \(b_{\epsilon}^{\star} \) as the full gradient norm $\| \nabla f (\bm{\theta}_t)\|$ decreases and reflects the empirically observed dynamics of SGD. It provides a principled mechanism for improving optimization efficiency without requiring manual tuning of the BS over time.

\subsection{Exponentially Increasing BS and LR Scheduler}
Proposition \ref{prop:1}(iii) implies that the upper bound of $\min_{t \in [0:T-1]} \mathbb{E}[\|\nabla f(\bm{\theta}_t)\|]$ decays at a rate of $O(1/\sqrt{\gamma^T})$ when the BS and LR are increased exponentially. A discussion analogous to that used to derive \eqref{eq:linear-scheduler}, together with the definitions of an exponentially increasing BS \eqref{bs} and CBS in \eqref{cbs} ($\epsilon = O (1/\sqrt{b_\epsilon^\star})$), leads to
\begin{align*}
    b_m = b_0 \delta^m = O (\delta^m) \text{ and }
    \epsilon_m = \frac{\epsilon_0}{\sqrt{\delta^{m}}}.
\end{align*}
We thus present a second candidate scheduler, with parameters $b_0$, $\delta$, $\eta_0$, and $\gamma$ as specified in \eqref{bs} and \eqref{lr}:
\begin{align}\label{eq:exp-scheduler}
    &\textbf{[Exponentially Increasing BS and LR Scheduler]}\nonumber\\
    &\quad \boxed{b_m = b_0  \delta^m, \text{ } \eta_m = \eta_0  \gamma^m, \text{ } \epsilon_m = \frac{\epsilon_0}{\sqrt{\delta^{m}}}.
    }
\end{align}
This exponentially increasing BS and LR 
scheduler adheres to the theoretical scaling law \(b_{\epsilon}^\star = O(1/\epsilon^2) \). The joint scheduling strategy couples the increases in BS and LR with the synchronized decay of the full gradient norm threshold. This preserves theoretical consistency and accelerates convergence compared with static or independently scheduled approaches.

We performed convergence analysis of mini-batch SGD with each of the two candidate schedulers:

\begin{proposition}\label{prop:3} 
Suppose the assumptions in Proposition \ref{prop:1} hold and that mini-batch SGD (Algorithm \ref{algo:1}) equipped with either candidate scheduler (\eqref{eq:linear-scheduler} or \eqref{eq:exp-scheduler}) achieves an $\epsilon_m$-approximation within $T_m$ steps. Then, for all $M$, 
    \begin{align*}
        \min_{t \in [0:T_{M-1} -1]} \mathbb{E} \left[ \|\nabla f (\bm{\theta}_t)\| \right]
        = 
        \begin{dcases}
            O \left(\frac{1}{\sqrt{M}} \right) \text{ } (\text{Scheduler }\eqref{eq:linear-scheduler})\\
            O \left(\frac{1}{\sqrt{\delta^M}} \right) \text{ } (\text{Scheduler } \eqref{eq:exp-scheduler}).
        \end{dcases}
    \end{align*}
\end{proposition}

\begin{proof}
Given that $\min_{t \in [0:T_{M-1} -1]} \mathbb{E} [ \|\nabla f (\bm{\theta}_t)\| ] \leq \epsilon_{M-1}$, \eqref{eq:linear-scheduler} and \eqref{eq:exp-scheduler} imply the result stated in Proposition \ref{prop:3}.
\end{proof}

To implement the two candidate schedulers in practice, we designed an adaptive algorithm that tracks the current stage $m$ and transitions to the next stage when the gradient norm drops below $\epsilon_m$. The complete procedures for \eqref{eq:linear-scheduler} and \eqref{eq:exp-scheduler} are summarized in Algorithm~\ref{algo:2}.

\begin{algorithm}[htbp]
\caption{Mini-batch SGD with adaptive schedulers}
\label{algo:2}
\begin{algorithmic}[1]
\REQUIRE
$\bm{\theta}_0 \in \mathbb{R}^d$ (initial point), 
$b_0 > 0$ (initial BS), 
$\eta_0 > 0$ (initial LR), 
$\epsilon_0 > 0$ (initial full gradient norm threshold),
$T \geq 1$ (max steps), 
$M \geq 1$ (total number of stages),
$\Delta b > 0$ (BS increase factor), 
$\gamma > 1$ (BS increase factor), 
$\delta > 1$ (LR increase factor),
\ENSURE 
$(\bm{\theta}_t) \subset \mathbb{R}^d$
\STATE $m \leftarrow 0$
\FOR{$t = 0,1,\ldots,T-1$}
    \STATE{ 
$\nabla f_{B_t}(\bm{\theta}_t)
:=
\frac{1}{b_m} \sum_{i=1}^{b_m} \nabla f_{\xi_{t,i}}(\bm{\theta}_t)$}
\STATE{
$\bm{\theta}_{t+1} 
:= \bm{\theta}_t - \eta_m \nabla f_{B_t}(\bm{\theta}_t)$}
\IF{$\|\nabla f(\bm{\theta}_t)\| \leq \epsilon_m$ \AND $m < M-1$}
        \STATE $m \leftarrow m + 1$

        \STATE $b_m = b_0 + m \Delta b$, $\eta_m = \eta_0$, $\epsilon_m = \frac{\epsilon_0}{\sqrt{1+m}}$ $\triangleleft$ \eqref{eq:linear-scheduler}
        \STATE $b_m = b_0 \delta^m, 
        \eta_m = \eta_0 \gamma^m, \epsilon_m = \frac{\epsilon_0}{\sqrt{\delta^{m}}}$ $\triangleleft$ \eqref{eq:exp-scheduler}
    \ENDIF
\ENDFOR
\end{algorithmic}
\end{algorithm}

\begin{figure}[htbp]
    \centering
        \includegraphics[width=0.9\linewidth]{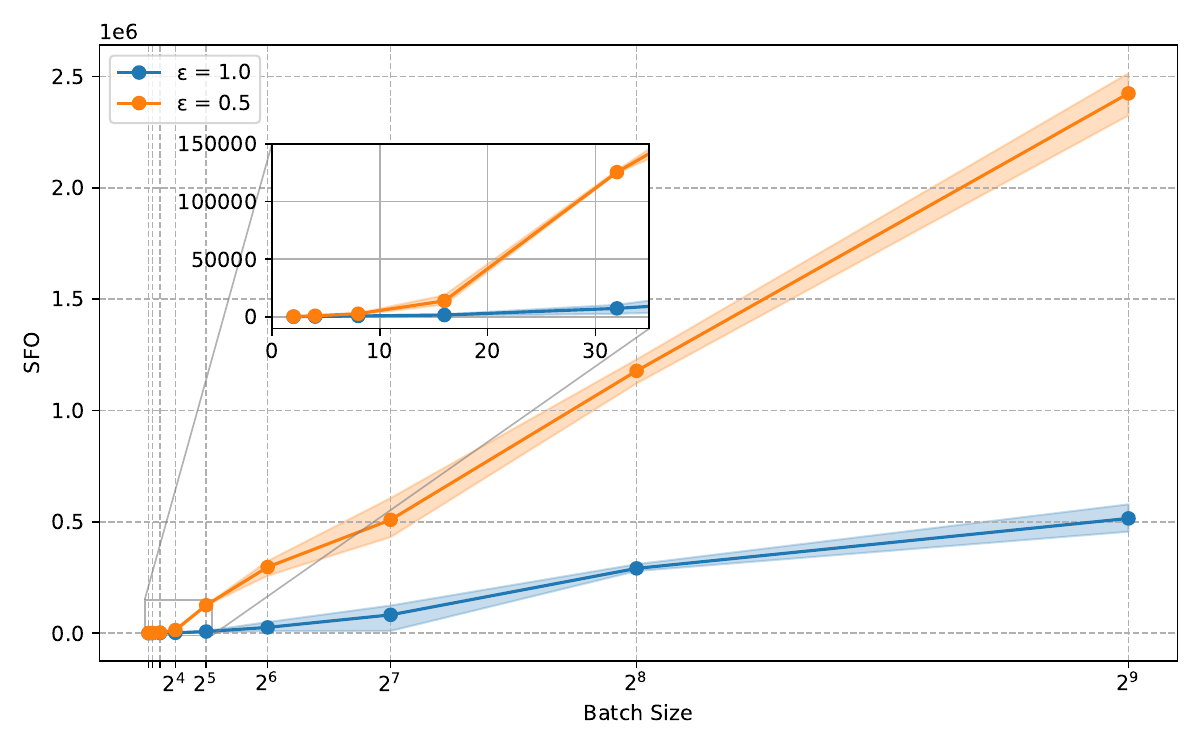}
\caption*{(a) ResNet-18 trained on CIFAR-10 with target precisions $\epsilon = 0.5$ and $\epsilon = 1$}
        \includegraphics[width=0.9\linewidth]{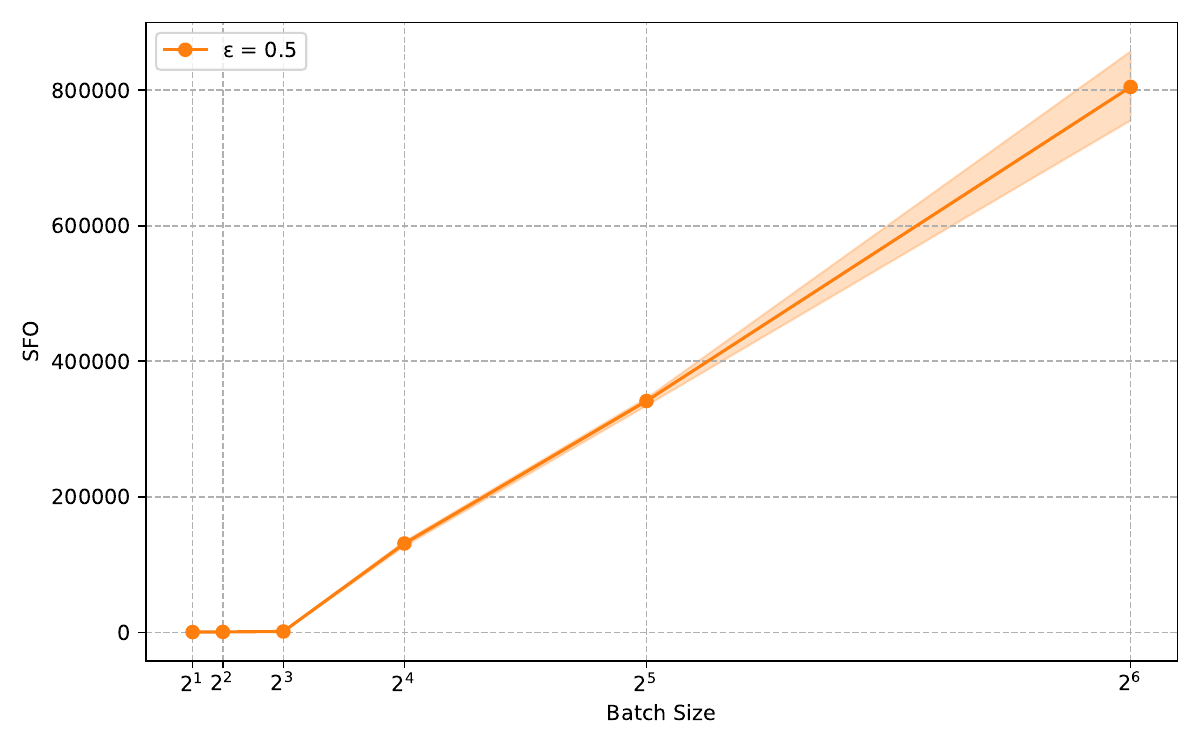}
\caption*{(b) DenseNet trained on CIFAR-100 with target precision $\epsilon = 0.5$ (CBS is not observed when $\epsilon = 1$)}
\caption{SFO complexity needed for SGD to achieve $\|\nabla f (\bm{\theta}_t)\| \leq \epsilon$ versus batch size.}
    \label{fig:cbs}
\end{figure}
\section{Evaluation}
To evaluate the performance of the two candidate schedulers, we performed experiments in which ResNet-18 was trained on CIFAR-10 and DenseNet was trained on CIFAR-100 using Algorithms~\ref{algo:1} and~\ref{algo:2}. All experiments were conducted on a system equipped with a NVIDIA A100 40-GB GPU and an AMD EPYC 7742 2.25-GHz CPU. The software stack comprised Python 3.10.12, PyTorch 2.1.0, and CUDA 12.2.
The solid lines in the figures represent the mean values, and the shaded areas in the figures indicate the maximum and minimum over three runs.

\subsection{Empirical Observation of CBS}
Figure~\ref{fig:cbs} illustrates the relationship between the BS and SFO complexity required to reach $\|\nabla f (\bm{\theta}_t)\| \leq \epsilon$ $(\epsilon = 0.5, 1)$ for ResNet-18 trained on CIFAR-10 and DenseNet trained on CIFAR-100. 
In both cases, the SFO curves exhibit a nearly convex shape and become approximately linear in the large-batch regime, consistent with the theoretical result in Proposition \ref{prop:2}. Notably, SFO complexity begins to increase almost linearly starting around a BS of $2^4 = 16$, suggesting that this value serves as the CBS in both settings. This supports the existence of a threshold beyond which increasing the BS yields diminishing returns in SFO efficiency.

\subsection{Comparison of Candidate Schedulers}
The performances of the two candidate schedulers (\eqref{eq:linear-scheduler} and \eqref{eq:exp-scheduler}) with $\epsilon_0 = 1$ (Figure~\ref{fig:cbs}) are compared in Figure~\ref{fig:compare-steps-linear}. The adaptive joint scheduler with exponentially increasing BS and LR achieved faster reduction in the full gradient norm across all stages. This behavior is consistent with the theoretical prediction in Proposition \ref{prop:3}, which states that adapting both BS and LR to the critical BS improves the convergence rate. 

\begin{figure}[htbp]
\centering
\includegraphics[page=1,width=\linewidth]{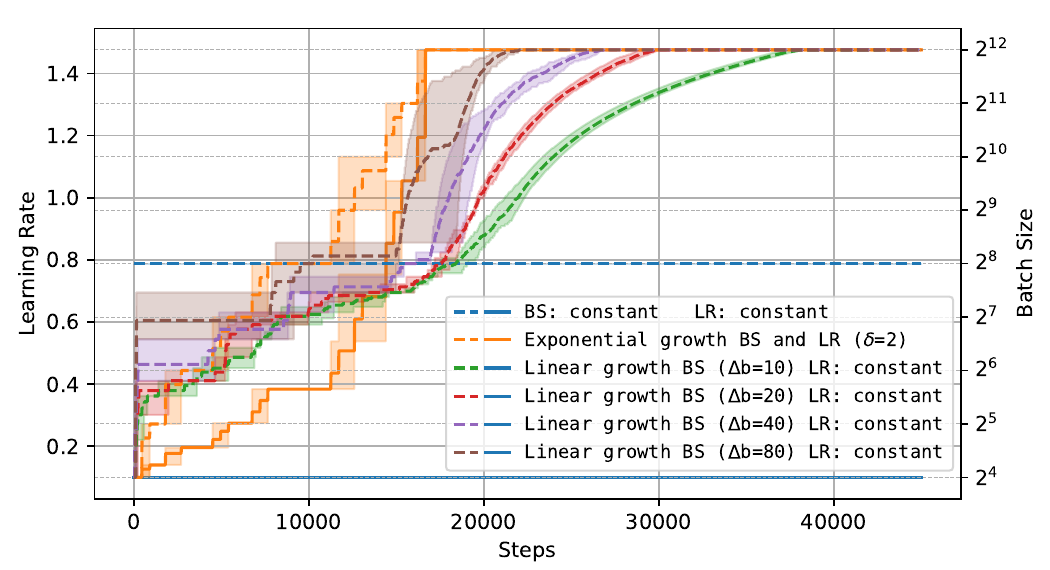}
\caption*{(a) Learning Rate}
\includegraphics[page=2,width=\linewidth]{figures/cifar10_compare_steps_proposed.pdf}
\caption*{(b) Full Gradient Norm of Empirical Loss for Training}
\includegraphics[page=3,width=\linewidth]{figures/cifar10_compare_steps_proposed.pdf}
\caption*{(c) Empirical Loss Value for Training}
\includegraphics[page=4,width=\linewidth]{figures/cifar10_compare_steps_proposed.pdf}
\caption*{(d) Accuracy Score for Testing}
\caption{Comparison of candidate schedulers in training ResNet-18 on CIFAR-10 dataset over 45k steps.}
\label{fig:compare-steps-linear}
\end{figure}

\subsection{Comparison with Existing Schedulers}

\begin{figure}[htbp]
  \centering
    \includegraphics[page=1,width=\linewidth]{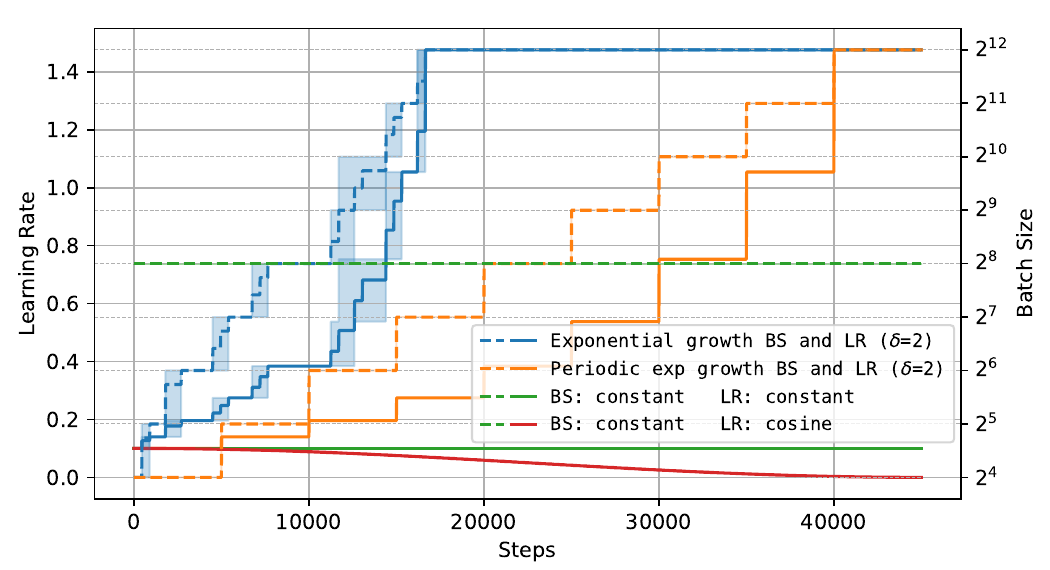}
 \caption*{(a) Learning Rate}
    \includegraphics[page=2,width=\linewidth]{figures/cifar10_compare_steps.pdf}
 \caption*{(b) Full Gradient Norm of Empirical Loss for Training}
    \includegraphics[page=3,width=\linewidth]{figures/cifar10_compare_steps.pdf}
 \caption*{(c) Empirical Loss Value for Training}
    \includegraphics[page=4,width=\linewidth]{figures/cifar10_compare_steps.pdf}
 \caption*{(d) Accuracy Score for Testing}
 \caption{Comparison of proposed adaptive joint scheduler with existing schedulers in training ResNet-18 on CIFAR-10 dataset over 45k steps.}
  \label{fig:compare_steps_cifar-10}
\end{figure}

\begin{figure}[htbp]
  \centering
    \includegraphics[page=1,width=\linewidth]{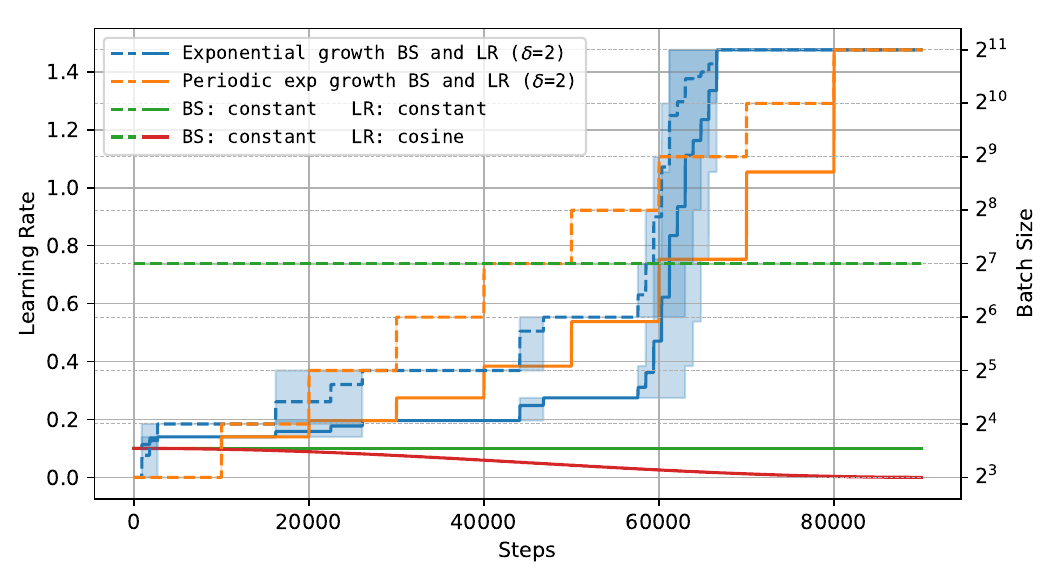}
 \caption*{(a) Learning Rate}
    \includegraphics[page=2,width=\linewidth]{figures/cifar100_compare_steps.pdf}
 \caption*{(b) Full Gradient Norm of Empirical Loss for Training}
    \includegraphics[page=3,width=\linewidth]{figures/cifar100_compare_steps.pdf}
 \caption*{(c) Empirical Loss Value for Training}
    \includegraphics[page=4,width=\linewidth]{figures/cifar100_compare_steps.pdf}
 \caption*{(d) Accuracy Score for Testing}
 \caption{Comparison of proposed adaptive joint scheduler with existing schedulers in training DenseNet on CIFAR-100 dataset over 90k steps.}
  \label{fig:compare_steps_cifar-100}
\end{figure}

The performance of the proposed adaptive joint scheduler is compared in Figures \ref{fig:compare_steps_cifar-10} and \ref{fig:compare_steps_cifar-100} against those of three existing schedulers: (i) a fixed BS and LR scheduler, (ii) a cosine annealing LR scheduler with a constant BS, and (iii) a fixed-interval update scheduler for both LR and BS (e.g., every 5,000 steps in Figure \ref{fig:compare_steps_cifar-10} and every 10,000 steps in Figure \ref{fig:compare_steps_cifar-100}). 

Figure \ref{fig:compare_steps_cifar-10} shows that the adaptive joint scheduler—where both BS and LR are increased on the basis of the full gradient norm—achieved the fastest convergence and the best overall performance. The fixed-interval update scheduler ranks second, highlighting the benefit of increasing both BS and LR. Figure \ref{fig:compare_steps_cifar-100} shows that the fixed-interval update scheduler performs comparably to the adaptive joint scheduler. However, unlike the adaptive method, it does not respond to the optimization dynamics. These results underscore the advantage of adapting the hyperparameters in response to the optimization landscape, particularly the gradient norm, rather than relying on predetermined schedules.

\section{Conclusion}
In our proposed adaptive scheduling strategy for mini-batch stochastic gradient descent, the batch size and learning rate are adjusted on the basis of the full gradient norm. Grounded in theoretical insights into the critical batch size and its relationship to the gradient norm threshold, our strategy provides a principled mechanism for dynamic hyperparameter tuning throughout training. Empirical and theoretical results demonstrate that the proposed adaptive joint scheduler 
accelerates convergence compared with existing approaches. These findings highlight the potential of leveraging optimization signals—such as the full gradient norm—for adaptive control of training dynamics. 
Future work includes extending this approach to other optimizers (e.g., Adam) and applying it to broader training scenarios.

\bibliography{aaai2026}


\end{document}